\pdfoutput=1

\documentclass[11pt]{article}

\usepackage{acl}

\usepackage{times}
\usepackage{latexsym}

\usepackage[T1]{fontenc}

\usepackage[utf8]{inputenc}

\usepackage{microtype}

\usepackage{inconsolata}

\usepackage{graphicx}
\usepackage{booktabs}
\usepackage{tabularx}
\usepackage{amsmath}
\usepackage{amsfonts}
\usepackage{mathtools}
\usepackage{hyperref}
\usepackage{cleveref}
\usepackage{multirow}
\usepackage{siunitx} 
\usepackage{caption}
\usepackage{amssymb}
\usepackage{amsthm}
\newtheorem{proposition}{Proposition}
\newtheorem{lemma}{Lemma}
\usepackage{xspace}
\usepackage[normalem]{ulem}

\title{Test-Time Steering for Lossless Text Compression via \\ Weighted Product of Experts}

\author{
 \textbf{Qihang Zhang\textsuperscript{1,2}}\quad \quad
 \textbf{Muchen Li\textsuperscript{1,2}}\quad \quad
 \textbf{Ziao Wang\textsuperscript{4}}\quad \quad
 \\
 \textbf{Renjie Liao\textsuperscript{1,2,3}}\quad \quad
 \textbf{Lele Wang\textsuperscript{1}}
\\
\\
 \textsuperscript{1}University of British Columbia \quad
 \textsuperscript{2}Vector Institute for AI 
 \\
 \textsuperscript{3}Canada CIFAR AI Chair
 \quad
 \textsuperscript{4}University of Michigan
\\
{\footnotesize
 \texttt{\{qihangz, rjliao, lelewang\}@ece.ubc.ca}
}\\
{\footnotesize
 \texttt{muchenli@cs.ubc.ca, ziaow@umich.edu}
}
}

\makeatletter
\DeclareRobustCommand\onedot{\futurelet\@let@token\@onedot}
\def\@onedot{\ifx\@let@token.\else.\null\fi\xspace}
 
\def\eg{\emph{e.g}\onedot} 
\def\ie{\emph{i.e}\onedot}

\makeatother

\begin{document}
\maketitle
\begin{abstract} 
Lossless compression techniques are crucial in an era of rapidly growing data. 
Traditional universal compressors like \texttt{gzip} offer low computational overhead, high speed, and broad applicability across data distributions. 
However, they often lead to worse compression rates than modern neural compressors, which leverage large-scale training data to model data distributions more effectively. 
Despite their advantages, neural compressors struggle to generalize to unseen data. 
To address this limitation, we propose a novel framework that performs Test-Time Steering via a Weighted Product of Experts (wPoE). 
At inference, our method adaptively combines a universal compression model with a pretrained neural language model, ensuring the compression rate is at least as good as that of the best individual model. 
Extensive experiments demonstrate that our approach improves the performance of text compression without requiring fine-tuning. 
Furthermore, it seamlessly integrates with any autoregressive language model, providing a practical solution for enhancing text compression across diverse data distributions. \textit{The code and all other details are available on our blog: \href{https://qihang-zhang.com/Learning-Sys-Blog/2025/10/15/weighted-product-of-experts.html}{blog for Weighted Product of Experts}.}
\end{abstract}
\section{Introduction}

Lossless text compression is a long-standing research challenge. Various general-purpose compression algorithms, \eg, \texttt{gzip}~\cite{gzip} and \texttt{LZMA2}~\cite{lzma2}, have been proposed to efficiently encode text data. 
These algorithms are widely adopted due to their low computational overhead, high efficiency, and consistent performance across diverse data distributions.

Shannon's source coding theorem \cite{shannon1948mathematical} establishes that the lower bound of average code length for a given dataset is determined by its Shannon entropy. 
While general-purpose compression algorithms are practical and efficient, their inability to model the underlying data distribution limits their ability to approach this bound. 
To overcome this limitation, neural network-based models trained to minimize the Kullback-Leibler (KL) divergence between the data distribution and the model's distribution have emerged as powerful alternatives~\citep{NNCP,lmic,llmzip}.

Empirical evidence suggests that neural compression techniques often outperform traditional methods, particularly when the input text closely resembles the data used for training. 
Large Language Models (LLMs)~\citep{gpt2,gpt4,llama3} have further advanced this capability, with studies such as \cite{lmic} and \cite{llmzip} demonstrating that leveraging LLMs can significantly enhance compression performance.

Despite these successes, neural compression models often struggle to generalize to  unseen data and require intensive computation~\citep{heurteldepeiges2024compressionpretrainedtransformersstudy}. In contrast, classical universal compressors, while not always achieving the highest compression ratios, maintain stable performance across diverse datasets without requiring training~\citep{lmic,heurteldepeiges2024compressionpretrainedtransformersstudy}. Additionally, these traditional methods typically have lower computational overhead.
\citet{mackay2003information,lmic} frame lossless compression and language modeling as two aspects of the same underlying principle, emphasizing the importance of developing more adaptable language models to improving text compression.

To address this limitation, we propose a novel framework within the Test-Time Steering (TTS) setting. 
Unlike methods that require fine-tuning on the target domain, TTS aims to adapt a pre-trained model to the target distribution, typically with a lightweight cost and a limited number of observed data points.
Specifically, we introduce a \emph{Weighted Product of Experts (wPoE)} method that dynamically combines a universal compression model with a pre-trained neural model. 
By selectively modulating each expert’s contribution based on the input text’s characteristics, \textit{wPoE} effectively adapts to shifts in data distribution. 

In summary, our contributions are as follow:
\begin{itemize}
    \item \textbf{Efficient Test-Time Steering for Generalizable Text Compression:}\
    We develop a weighted product of experts (wPoE) framework that combines a training-free universal compressor, \ie, Naive Bayes with Laplace smoothing, with an autoregressive language model to enable generalizable text compression at test time. 
    The optimal expert weights are efficiently determined by optimizing a single scalar using one data point from the text to be compressed at test time.
    \item \textbf{Theoretical Guarantee for wPoE:}\
    We provide a theoretical proof that wPoE performs at least as well as the best individual expert in text compression.
    \item \textbf{Improved Empirical Performance on Text Compression:}\
    Since our framework seamlessly integrates with any autoregressive language model without requiring fine-tuning, we conduct extensive experiments across various text datasets and LLMs. 
    We demonstrate that our method consistently improves compression rates across a wide range of pretrained LLMs.
    Additionally, our method uses less computation and GPU memory compared to fine-tuning, making it ideal for environments with limited resource budgets.
\end{itemize}
\section{Related Work}
\paragraph{Universal Compression}
Lossless data compression is a central topic of information theory. A data sequence is assumed to be generated by a random process. The goal is to design lossless compressors and decompressors that achieve the theoretical limit, i.e., the entropy rate of the data, asymptotically as the length of the data sequence goes to infinity. When the underlying distribution/statistics of data is known, optimal lossless compression can be achieved by methods like Huffman coding. However, in most real-world applications, the exact distribution is usually hard to obtain and the data we are given is a single realization of this distribution. This motivates the framework of \emph{universal compression}, in which we assume the underlying distribution belongs to a known family of distributions and require that the compressor and the decompressor should not be a function of the underlying distribution. The goal of universal compression is to design a single compression scheme that universally achieves the optimal theoretical limit, for every distribution in the family, without knowing which distribution generates the data.

Over the past five decades, various universal compressors have been developed. For the family of independent and identically distributed sequences, the Laplace compressor~\cite{Laplace1995} and the Krichevsky--Trofimov (KT) compressor~\cite{Krichevsky--Trofimov1981} are known to be universal. For the family of stationary ergodic processes, the Lempel--Ziv compressors~\cite{Lempel--Ziv77,Lempel--Ziv78} and the Burrows--Wheeler transform~\cite{effros2002universal} achieve the optimal performance.
For finite memory processes, techniques such as context tree weighting~\cite{willems1995context} have been developed. These universal compression techniques have found applications beyond theoretical research, playing a crucial role in standard data compressor such as \texttt{gzip} \cite{gzip}, image formats like \texttt{GIF}\cite{gif} and \texttt{TIFF} \cite{TIFF6}, and file compressors like \texttt{bzip2} \cite{bzip2}.

While most of the universal compressors can be shown to achieve the entropy rate in the asymptotic limit, the speeds they converge to the theoretical limit are known to be slow, even for variations of these algorithms that are highly optimized and widely adopted in practical file compression.

\paragraph{Lossless Compression with Language Models}
In recent years, the increasing availability of data has given rise to a new compression paradigm. This approach assumes access not only to the data to be compressed but also to additional training data. In case the training data is generated from the same distribution, a neural network can be trained on this auxiliary data to capture the underlying distribution. Compressors are then designed based on the learned distribution, leveraging the well-established equivalence between prediction and compression~\citep{schmidhuber1994predictive, schmidhuber1996sequential, mahoney2000fast, mikolov2012statistical, knoll2014cmix, cox2016syntactically, goyal2019deepzip, liu2019decmac, lmic, llmzip}.
Recent progress in large language models \cite{llama3, gpt4} has highlighted their growing proficiency in next-token prediction. 
Building on these advances, \citet{llmzip} explore the use of LLMs for lossless text compression, showing that their predictive abilities can effectively encode information. Concurrently, \citet{narashiman2024alphazip} show that domain-adaptive fine-tuning with GPT-2 Small improves compression efficiency over GZIP. Furthermore, \citet{mittu2024finezip} introduces an online memorization module to enhance compression efficiency.
Additionally, \citet{lmic} extend the use of LLM-based compression beyond text, demonstrating their potential for lossless compression for other modalities. In this paper, we consider this compression paradigm and focus on the case where the training data and testing data may come from different distributions. We demonstrate that existing neural compressors can be further improved by exploiting ideas from the product of experts introduced next.

\paragraph{Product of Experts} First introduced by \citet{PoE}, the \textit{Product of Experts (PoE)} model multiplies the probabilities output by \(K\) probabilistic models and then normalizes them over the entire dictionary, thereby producing a valid probability distribution.
\citet{gPoE} further extends this idea. An adaptive exponent is applied to each expert’s output probability, allowing the model to dynamically adjust the influence of each expert on the final output distribution.
In the context of text generation,  \citet{liu2021dexperts,hallinan2023steer} use the products of three experts to de-toxic the output for base language models. 
\section{Background}
In this section, we explain the equivalence between prediction and compression in more detail and conclude that the problem of lossless compression boils down to the design of a probability distribution for the given data that is as close to the true data distribution as possible.

\paragraph{Equivalence between Prediction and Compression}Let $\mathcal{A} = \{a_1, a_2, \ldots, a_D\}$ be a vocabulary set of size $D$, and $X_{<n+1} := \{X_1, X_2, \ldots, X_n\}\in\mathcal{A}^n$ denote a random sequence that follows the  probability distribution $p_{\text{data}}$. Let $x_{<n+1}:= \{x_1, x_2, \ldots, x_n\}$  denote a realization of $X_{<n+1}$.
Suppose we assign probability $p_\theta(X_{<n+1})$ to the sequence via arithmetic coding~\citep{pasco1977source, rissanen1976generalized}. It can be shown that~\cite{cover1999elements} the expected compression length can be bounded as $H(p_{\text{data}}, p_{\theta}) \le L_{p_{\theta}} \le H(p_{\text{data}}, p_{\theta})+2$, where $H(p_{\text{data}}, p_{\theta}) := - \mathbb{E}_{p_{\text{data}}} \left[ \log {p_{\theta}(X_{<n+1})} \right]$ denotes the cross-entropy between $ p_{\text{data}}$ and $p_{\theta}$. This implies that the closer $p_{\theta}$ to $ p_{\text{data}} $, the smaller the expected code length $ L_{p_{\theta}}$ becomes, until reaching the minimum, \ie, the Shannon entropy of $p_{\text{data}}$ when $ p_{\theta}= p_{\text{data}}$.

\paragraph{LLM-based compressor}  
\label{subsect:ac_autoregressive_model}
Given that LLMs are autoregressive models, we can exploit the chain rule:  
$
p_\theta(X_{<n+1}) = \prod_{i=1}^n p_\theta(X_i \mid X_{<i}).
$  
This means that compressing the entire sequence reduces to encoding each token sequentially using the conditional probabilities $p_\theta(X_i \mid X_{<i})$ for $i = 1, \ldots, n$, which is naturally compatible with arithmetic coding~\citep{pasco1977source, rissanen1976generalized}. As shown in Figure~\ref{fig:pipeline}, when compressing each token in a sequence, arithmetic coding takes the corresponding categorical distribution, selects the appropriate interval, and encodes the sequence into a real number. A detailed explanation of arithmetic coding is provided in Appendix~\ref{sec:detailed_intro_to_ac}.
\section{Method}
In this section, we propose a novel lossless text compressor that is robust and generalizable. We start by explaining how we design the conditional probabilities $p_\theta(X_i|X_{<i})$ based on a variation of Product of Experts.
\begin{figure*}
    \centering\includegraphics[width=0.85\linewidth]{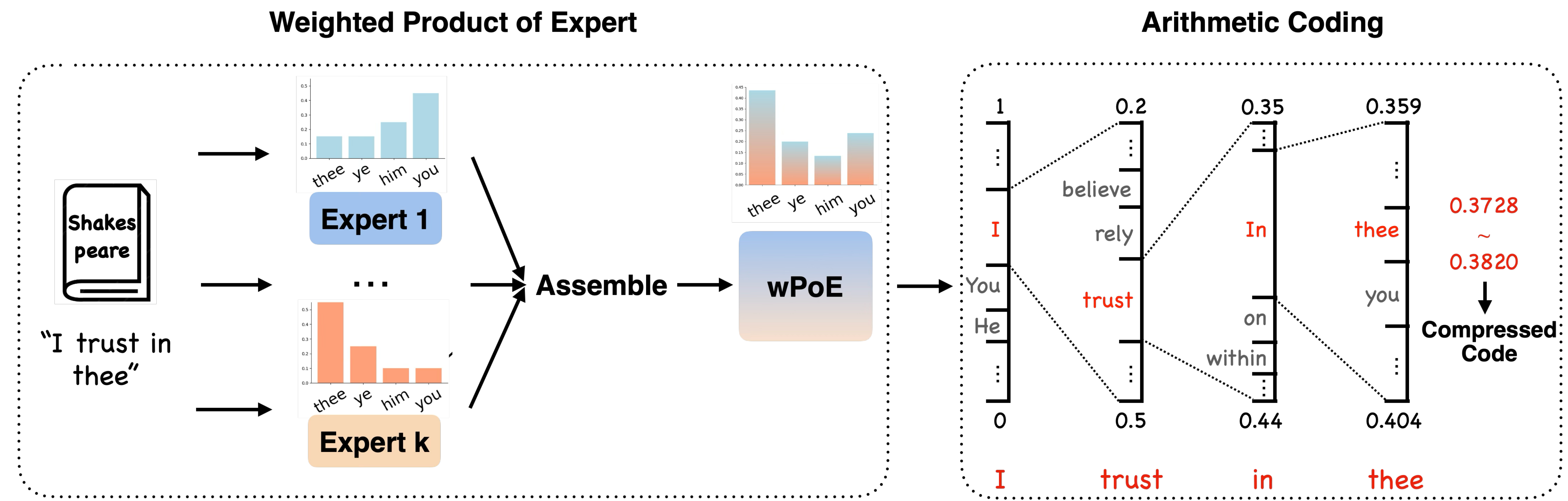}
    \caption{A diagram of our compression pipeline: we apply a weighted-product-of-experts approach to steer the probability distribution during inference. After obtaining the steered distribution, we use arithmetic coding to compress the sequence. As shown in the figure, the sequence ‘I, trust, in, thee’ can be compressed as any real number in the interval [0.3728, 0.3820), as introduced in Section~\ref{subsect:ac_autoregressive_model}.}
    \label{fig:pipeline}
\end{figure*}

\subsection{Weighted Product of Experts (wPoE)}
PoE \citep{PoE}  and gPoE \citep{gPoE} have been proposed before as better models that can be trained from scratch. 
Here, we propose a test-time steering approach for pretrained models using the weighted product of experts framework to improve performance.

Suppose we have $K$ experts $p_{\theta_1},p_{\theta_2},...,p_{\theta_K}$, \eg, autoregressive models, and we want to compress the data $X_{<n+1}$ that follows $p_{\text{data}}$. 
Our weighted product of expert (wPoE) model is given as follows,
\begin{equation}\label{eq:wpoe}
\begin{array}{cc}
    p_{\boldsymbol{\theta}, \boldsymbol{\alpha}}(X_n|  X_{<n}) =  \frac{1}{Z(\boldsymbol{\theta}, \boldsymbol{\alpha},n)}\displaystyle\prod_{k = 1}^K p_{\theta_k}(X_n  |  X_{<n})^{\alpha_k},
\end{array}
\end{equation}
where the weights are $\boldsymbol{\alpha} = \{\alpha_1,...,\alpha_K\}$, $\alpha_k \in[0,1]$, $\sum_{k = 1}^K \alpha_k = 1$, the parameters of experts are $\boldsymbol{\theta} = \{\theta_1,...,\theta_K\}$, and the normalization constant is $Z(\boldsymbol{\theta}, \boldsymbol{\alpha}, n) = \sum_{a \in \mathcal{A}}\prod_{k = 1}^K p_{\theta_k}(X_n = a |  X_{<n})^{\alpha_k}$.

We have the following theoretical result, which demonstrates that the optimal weighted product of experts performs at least as well as the best individual expert.
\begin{proposition}
\label{prop:wPoe_ensures_a_better_model}
Given the weighted product of expert model in Eq. \eqref{eq:wpoe}, we have 
\begin{equation*}
    \displaystyle\inf_{\boldsymbol{\alpha}}H(p_{\text{data}},p_{\boldsymbol{\theta}, \boldsymbol{\alpha}} ) \leq \displaystyle\min_{k \in \{1,...K\}} H(p_{\text{data}}, p_{\theta_k}).
\end{equation*}
\end{proposition}

It is crucial to emphasize that Proposition \ref{prop:wPoe_ensures_a_better_model} cannot hold without the constraint $\sum_{k = 1}^K \alpha_k = 1$ in our setting, where we do not update the parameters of the models. This is because the cross-entropy of a wPoE model can be decomposed into a weighted average of each expert’s cross-entropy (using the weights \(\alpha_k\)) plus an additional term. The condition \(\sum_{k=1}^K \alpha_k = 1\) is necessary for this additional term to become negative, allowing the wPoE’s overall cross-entropy to potentially be lower than that of any single expert. The detailed proof is given in Appendix \ref{sec:proof_of_proposition}. 

\subsection{Test-time Steering for Generalizable Compression}\label{subsect:test_time_steering}

Suppose we have a language model, such as an autoregressive model, that has been pretrained on a dataset following the distribution $p_{\text{data}}$.
Now, we aim to use this model for compressing a new dataset that follows a different distribution $p_{\text{data}}^{\prime}$.
Instead of fine-tuning the pretrained model on the new dataset—which can be computationally expensive—we seek to perform test-time steering, i.e., directly adjusting the pretrained model during inference.
In particular, we achieve test-time steering using the weighted Product of Experts (wPoE) model, where we combine the pretrained model with a training-free, lightweight model—specifically, a Naive Bayes classifier with Laplace smoothing ~\citep{naive_bayes}.

As proposed by \citet{zipf1949human} in \textit{Zipf's law}, the frequency of a word's occurrence in a text is inversely proportional to its rank in the frequency table. Based on this prior knowledge, Naive Bayes classifier with Laplace smoothing~\citep{naive_bayes} is widely adopted as the prior distribution for entropy coding in lossless text compression~\cite{Laplace1995}. 
It defines the following conditional probability of the next token $X_n$ conditioned on the previous tokens $X_{<n}$:
\begin{equation}
\label{eq:nb_prob}
q(X_n = a \mid X_{<n}) \coloneqq \frac{\sum_{k=1}^{n-1} \mathbb{I}(X_k = a) + 1}{n - 1 + D},
\end{equation}
where $\mathbb{I}(\cdot)$ denotes the indicator function and $D$ is the vocabulary size. 
Since Naive Bayes with Laplace smoothing has no learnable parameters, we can treat it as a training-free expert.
Despite its simplicity, this Laplace smoothing prior performs well across various types of text data, making it a strong candidate as a steering expert to assist pretrained autoregressive models in text compression. 
We then combine the Naive Bayes with Laplace smoothing $q$ with a pretrained language model $p_{\theta}$ using the weighted product of experts as follows, 
\begin{equation}\label{eq:wpoe_our}
    \pi_{\alpha}\!(X_n \vert X_{<n}) \!= \!
    \frac{q(X_n \vert X_{<n})^{\alpha} p_{\theta}(X_n \vert X_{<n})^{1 - \alpha}}
    {Z(\theta, \alpha, n)}.
\end{equation}
Here $\alpha$ is a scalar as we have only two experts.
Moreover, since we do not need to fine-tune the pretrained model $p_{\theta}$, \ie, $\theta$ is frozen, we omit the dependency of $\theta$ in the wPoE model $\pi$.

\subsection{Learning the Optimal Weights}\label{subsect:optimal_weights}

Although we obtain a wPoE model for test-time steering in Section \ref{subsect:test_time_steering}, there is no guarantee that it achieves better compression than the individual experts, \ie, the pretrained language model and the Naive Bayes model with Laplace smoothing~\citep{naive_bayes}. Proposition \ref{prop:wPoe_ensures_a_better_model} suggests that achieving better compression with wPoE requires determining the optimal weights ${\alpha}^{\ast}$.

A straightforward approach is to perform a grid search over all possible values of \({\alpha}\), but this is computationally expensive. 
Instead, we directly optimize the following objective with respect to \({\alpha}\) using gradient-based methods:
\begin{equation}
\min_{{\alpha}} \quad H(p_{\text{data}}^{\prime},\pi_{{\alpha}} ).
\end{equation}

Empirically, we find that even with a single data point and only 10 iterations of a second-order optimizer L-BFGS \cite{lbfgs}, this procedure yields a near-optimal solution for the weights. 
As the optimization is performed over a single scalar \(\alpha\), its computational overhead is negligible compared to the cost of model inference.

\subsection{Arithmetic Coding with wPoE}
After learning the optimal weight \(\alpha^{\ast}\), we can compress the new text data using our wPoE model.
Specifically, since our wPoE model in Eq. \eqref{eq:wpoe_our} is still an autoregressive model, we can use its conditional probabilities \(\pi_{\alpha^{\ast}}\!(X_i \vert X_{<i})\) to construct the arithmetic code~\citep{pasco1977source, rissanen1976generalized}.
This allows us to compress a sequence into a binary representation of a decimal number between 0 and 1.
Further details on adaptive arithmetic coding are provided in Appendix \ref{sec:detailed_intro_to_ac}.

\begin{table*}[htb]
  \centering
  \caption{The table presents the compression rates for five datasets. An asterisk (*) indicates that enwik8 and enwik9 are considered in-distribution for vanilla transformers. Additionally,  all “pretrained model + Ours” refers to the ensemble comprising the pretrained model and a Naive Bayes classifier with Laplace smoothing. We report the mean results over 5 runs. 
}
  \label{tab:compression_results}
    \resizebox{0.9\linewidth}{!}{
    \begin{tabular}{llcccll}
        \toprule
        \textbf{Tokenizer}& \textbf{Compressor} & \textbf{math} & \textbf{code} & \textbf{shakespeare}  & $\textbf{enwik8}^*$ &$\textbf{enwik9}^*$ \\
        \midrule
        \multirow{3}{*}{Byte Level}
        & gzip  &  43.59\%&  36.72\%& 52.80\% & 49.14\%&48.07\%\\
        & LZMA2 &  45.35\%&  38.61\%& 56.86\% & 51.33\%&49.98\%\\
        & Naive Bayes & 68.90\% & 64.65\% & 64.57\% & 66.03\% &67.14\% \\
        \cmidrule(lr){2-7}
        & Transformer 200K & 56.25\% & 65.67\% & 44.04\%  & 31.59\% &30.74\% \\
        & Transformer 200K + Ours& \textbf{50.95\%} & \textbf{53.94\%} & \textbf{42.12\%}  & \textbf{31.58\%} & \textbf{30.71\%} \\
        \cmidrule(lr){2-7}
        & Transformer 800K & 47.41\% & 62.13\% & 40.53\%  & 25.97\% &25.52\% \\
        & Transformer 800K + Ours& \textbf{44.34\%} & \textbf{49.68\%} & \textbf{38.79\%}  & \textbf{25.94\%} & \textbf{25.45\%} \\
        \cmidrule(lr){2-7}
        & Transformer 3.2M & 34.15\% & 41.02\% & 32.02\%& 18.53\% &17.66\% \\
        & Transformer 3.2M + Ours& \textbf{32.04\%} & \textbf{36.61\%} & \textbf{31.29\%}  & \textbf{18.52\%} & \textbf{17.65\%} \\
        \midrule
        BPE Tokenizer& Naive Bayes & 66.41\% & 59.30\% & 49.74\%  & 48.85\% &53.43\% \\
        \cmidrule(lr){2-7}
        (GPT-2)& GPT-2 & 17.68\% & 14.17\% & 23.44\%  & 16.48\% &16.73\% \\
        & GPT-2 + Ours& \textbf{17.55\%} & \textbf{14.16\%} & \textbf{23.11\%} & \textbf{16.42\%} & \textbf{16.65\%} \\
        \midrule
        BPE Tokenizer& Naive Bayes & 68.70\%& 47.54\% & 51.35\% & 48.87\% &51.93\%\\
        \cmidrule(lr){2-7}
        (LLaMA 3)& LLaMA 3.2-1B& 8.54\%& 6.66\%& 16.51\% & 10.22\%&10.05\%\\
         & LLaMA 3.2-1B + Ours& \textbf{8.48\%}& \textbf{6.64\%}& \textbf{16.42\%} & \textbf{10.16\%}& \textbf{9.98\%} \\
        \cmidrule(lr){2-7}
        & LLaMA 3.2-3B& 7.56\%& 5.99\%& 13.97\% & 9.16\%&8.93\%\\
        & LLaMA 3.2-3B + Ours& \textbf{7.50\%}& \textbf{5.95\%}& \textbf{13.88\%} & \textbf{9.09\%}& \textbf{8.86\%} \\
        \cmidrule(lr){2-7}
        & LLaMA 3-8B&  6.90\%& 5.61\% & 4.74\% & 8.18\% &8.10\%\\
        & LLaMA 3-8B + Ours& \textbf{6.84\%}& \textbf{5.57\%} & \textbf{4.73}\% & \textbf{8.12\%} & \textbf{8.04\%} \\
        \bottomrule
    \end{tabular}
    }
\end{table*}
\section{Experimental Evaluation}
\label{sec:results}

In this section, we evaluate the effectiveness of our framework through text compression experiments on diverse text datasets using various large language models (LLMs).

\paragraph{Language Models.} We evaluate two categories of pretrained language models. 
Following \citet{lmic}, the first category consists of byte-level decoder-only Transformers of varying sizes—200k, 800k, and 3.2M—trained on the \texttt{enwik8} dataset~\citep{vaswani2017attention}. 
The second category includes pretrained open-source Large Language Models, such as GPT-2~\citep{gpt2} and Llama-3~\citep{llama3}. 
For comparison, we also report the performance of widely used universal compressors, including gzip~\citep{gzip} and LZMA2~\citep{lzma2}.

\paragraph{Datasets.} We consider a wide range of text datasets: \texttt{enwik8}, \texttt{enwik9}, \texttt{code}, \texttt{math}, and \texttt{shakespeare}. \texttt{enwik9}~\citep{hutter2006prize} consists of the first $10^9$ bytes 
of the English Wikipedia dump on March~3, 2006. 
Meanwhile, \texttt{enwik8} is the first one-tenth portion of \texttt{enwik9}. The \texttt{code} dataset, introduced by \citet{bigcodebench}, is a code generation benchmark comprising 1,140 fine-grained Python programming tasks, spanning 139 libraries across 7 domains, which ensures that the dataset is sufficiently diverse in the coding domain. The \texttt{math} dataset, published by \citet{mathdataset}, contains 12,500 challenging competition mathematics problems spanning seven subjects, \eg, algebra, geometry, and number theory. \texttt{Shakespeare} \cite{shakespeare2007complete} contains the complete works, plays, sonnets, and poems of William Shakespeare. Following the setting of \citet{lmic}, we partition all datasets into sequences of 2048 bytes to optimize inference efficiency in Transformer-based models, such as GPT-2 and Llama-3.

\subsection{Test-time Steering for Text Compression With Two Experts}
In this section, we present the performance of the wPoE framework, which combines two experts: a pretrained model and a Naive Bayes classifier with Laplace smoothing. 
In Table \ref{tab:compression_results}, we compare the performance of the wPoE model with the original pretrained model across five datasets: \texttt{enwik8}, \texttt{enwik9}, \texttt{code}, \texttt{math}, and \texttt{shakespeare}.

It is important to note that byte-level decoder-only Transformers are trained on \texttt{enwik8}, and that \texttt{enwik8} and \texttt{enwik9} share closely related text distributions. 
Therefore, both \texttt{enwik8} and \texttt{enwik9} are considered in-distribution datasets for the three byte-level decoder-only Transformers.

Our results demonstrate that incorporating the Naive Bayes classifier with Laplace smoothing into byte-level decoder-only Transformers via wPoE significantly improves performance on out-of-distribution (OOD) datasets. 
The most notable improvement lies in Transformer 200K on the \texttt{code} dataset, likely due to the fact that the training corpus (\texttt{enwik8}) predominantly consists of natural language, which differs significantly from code. 
In contrast, \texttt{math} and \texttt{shakespeare} contain more natural language content, resulting in smaller, yet still notable, performance gains. 
Even on in-distribution datasets like \texttt{enwik8} and \texttt{enwik9}, we observe minimal but consistent improvements.

As model size increases, the benefit of wPoE decreases across all datasets. 
This trend is also observed with large language models (LLMs), where the improvement from wPoE is smaller compared to the pretrained byte-level Transformers.

We also evaluate our wPoE-based approach on GPT-2 and LLaMA 3 models of various sizes (1B, 3B, and 8B). 
Although large-scale LLMs are typically trained on diverse and extensive text corpora, resulting in minimal distribution shift with potentially unseen data, our results still show consistent improvements across various datasets and model sizes.
These findings highlight the versatility and effectiveness of our approach, demonstrating its applicability not only to smaller Transformers but also to state-of-the-art open-source LLMs.

\subsection{Test-time Steering for Generalizable Compression With Multiple Experts}
In this section, we conduct experiments with multiple experts. 
Specifically, we leverage byte-level decoder-only Transformers pretrained on the \texttt{enwik8} dataset, including a Transformer with 200k parameters, a Transformer with 800k parameters, a Transformer with 3.2M parameters, and a Naive Bayes classifier with Laplace smoothing. 
We progressively incorporate additional experts into the 3.2M-parameter Transformer using the Weighted Product of Experts (wPoE) approach and evaluate the assembled wPoE model on three OOD datasets.

Table~\ref{tab:muti_experts_compression_results} shows the reduction in compression rate achieved by wPoE models consisting of two, three, and four experts, compared to the 3.2M-parameter Transformer baseline. 
The results demonstrate that as the number of experts increases, the performance of the wPoE model improves steadily. 
Notably, the inclusion of the Naive Bayes classifier with Laplace smoothing results in a significantly higher improvement in compression rate than the addition of a smaller Transformer. 
This effect can be attributed to two factors: first, all three Transformer models are trained on the same dataset, and second, Laplace smoothing serves as an effective universal prior for text distributions.

\begin{table}[!htb]
  \centering
  \caption{Compression rates on three OOD datasets \texttt{code}, \texttt{math}, and \texttt{shakespeare}, as more experts are added to the Transformer 3.2M model. "1 expert" refers to the base Transformer pretrained on \texttt{enwik8}, additional experts correspond to the inclusion of 800K, 200K, and a Naive Bayes expert, respectively.}
  \label{tab:muti_experts_compression_results}
  \resizebox{0.85\linewidth}{!}{
    \begin{tabular}{lccc}
      \toprule
      \textbf{Compressor} & \textbf{math} & \textbf{code} & \textbf{shakespeare} \\
      \midrule
      1 expert & 34.15\% & 41.02\% & 32.02\% \\
      2 experts wPoE& 33.63\% & 40.59\% & 31.99\% \\
      3 experts wPoE& 33.62\% & 40.46\% & 31.97\% \\
      4 experts wPoE& \textbf{31.99\%} & \textbf{36.49\%} & \textbf{31.35\%} \\
      \bottomrule
    \end{tabular}
  }
\end{table}
\begin{figure*}[!ht]
    \centering
    \includegraphics[width=0.95\linewidth]{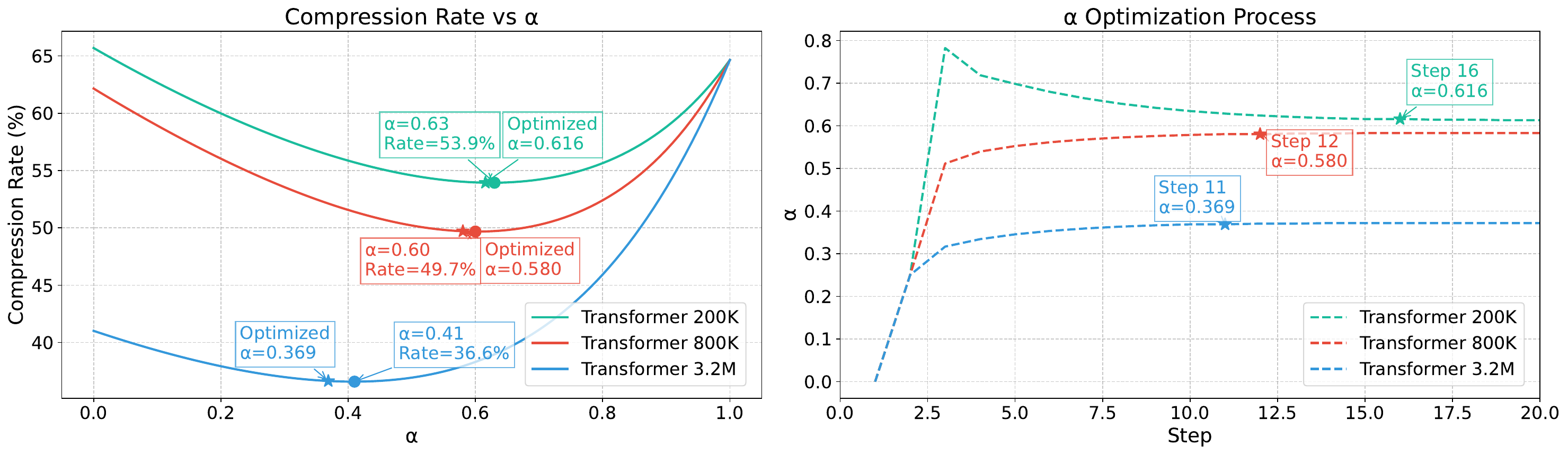}
    \caption{Figure (a) illustrates the compression rate on the OOD dataset \texttt{code} for three Our wPoE models of different sizes, under varying values of $\alpha$. Figure (b) shows $\alpha$ changes as the number of iterations increases when using L-BFGS optimizer. The annotated circular points represent the optimal $\alpha$ found through grid search. The points marked with asterisks indicate the converged $\alpha$ values after optimization.}
    \label{fig:alpha_optimize_vs_gridsearch}
    \vspace{-1em}
\end{figure*}

\subsection{Effectiveness of Optimal Weight Search}
In this section, we demonstrate the effectiveness of finding the optimal $\alpha$ via optimization and provide experimental observations that help explain this effectiveness. 
The results are shown in Figure~\ref{fig:alpha_optimize_vs_gridsearch}.

In the Figure~\ref{fig:alpha_optimize_vs_gridsearch}(a), we show how the compression rate of wPoE on the \texttt{code} dataset varies with $\alpha$ for byte-level decoder-only Transformers of three different sizes, each pretrained on the \texttt{enwik8} dataset. 
Notably, when $\alpha = 0$ or $\alpha = 1$, the wPoE method degenerates into a single expert.
Specifically, when $\alpha = 0$, the resulting compression rate corresponds to the performance of the pretrained byte-level decoder-only Transformers, while when $\alpha = 1$, it reflects the performance of the Naive Bayes classifier with Laplace smoothing. 
We employ grid search to determine the best $\alpha$ for each model.

Figure~\ref{fig:alpha_optimize_vs_gridsearch}(b) illustrates the evolution of $\alpha$ over multiple iterations in the optimization process on a single sample in the \texttt{code} dataset. 
In this setting, we selected a single OOD data sample, \ie, a sequence occupying 2048 bytes, from the new dataset and maximized the log probability of this sample to update $\alpha$ via backpropagation. 
Our results show that even with a single OOD data point, a second-order optimizer like L-BFGS \cite{lbfgs}) allows fast convergence of $\alpha$ within 20 iterations, yielding a near-optimal value.

Moreover, we observe that for a wPoE consisting of two experts, when using a grid search with a 0.01 interval, the observed compression rate on the potentially unseen dataset as a function of $\alpha$ is roughly convex. 
This phenomenon is consistently observed across all wPoE models in Table~\ref{tab:compression_results} that combine pretrained models with the Naive Bayes classifier with Laplace smoothing. 
This observation explains why, when optimizing $\alpha$ using a second-order optimizer like L-BFGS, $\alpha$ converges rapidly and accurately.

Finally, we note that when using a first-order optimizer, such as Adam~\citep{adam}, to optimize $\alpha$, an excessively high learning rate causes $\alpha$ to oscillate systematically over successive iterations. 
The amplitude of these oscillations gradually decreases as the iterations progress.

\begin{table}[!t]
  \centering
 \caption{Adaptation cost to OOD datasets in terms of computation and memory. Compared with finetuning, our method achieves competitive compression performance with significantly lower computational overhead. Results are averaged over 5 runs.}
  \label{tab:efficiency_table}
  \small
  \resizebox{\linewidth}{!}{
    \begin{tabular}{llccc}
        \toprule
        \textbf{Model} &  & \textbf{Compression} & \textbf{Computation} & \textbf{GPU Memory} \\
        \textbf{Size} & & \textbf{Rate} & \textbf{(GFLOPS)} & \textbf{Usage} \\
        \midrule
        \multirow{3}{*}{200K} & Pretrained & 65.74\% & N/A& N/A\\
         & Ours & \textbf{56.37\%} & \textbf{20.7} & \textbf{800M} \\
         & Finetune & 61.80\% & 31.0 & 1592M\\
        \midrule
        \multirow{3}{*}{800K} & Pretrained & 62.19\% & N/A& N/A\\
         & Ours & \textbf{52.73\%} & \textbf{47.8} & \textbf{806M} \\
         & Finetune & 53.87\% & 71.6 & 1616M\\
        \midrule
        \multirow{3}{*}{3.2M} & Pretrained & 41.08\% & N/A& N/A\\
         & Ours & 37.15\% & \textbf{151.7} & \textbf{826M} \\
         & Finetune & \textbf{35.45\%} & 182.0 & \textbf{1664M}\\
        \bottomrule
    \end{tabular}
  }
\end{table}
\begin{figure}[!t]
    \centering
    \includegraphics[width=\linewidth]{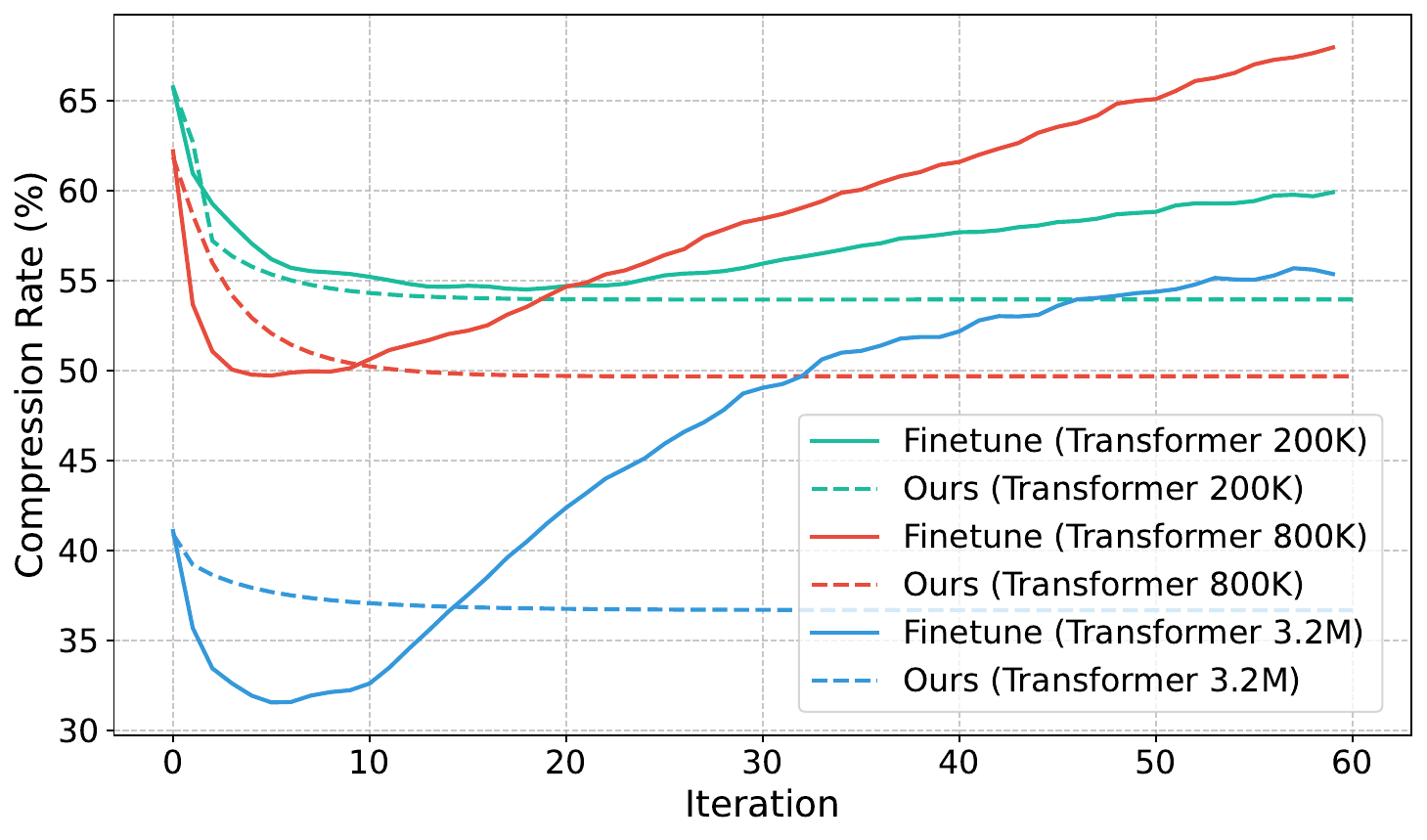}
    \caption{We compare our method against fine-tuning on byte-level decoder-only Transformers, given a sequence containing 2048 bytes of out-of-distribution data. Our method is more robust to overfitting. Consistent with Figure~\ref{fig:alpha_optimize_vs_gridsearch} (b), we initialize $\alpha$ to be 0.0.}
    \label{fig:optimize_vs_finetune}
    \vspace{-1em}
\end{figure}

\subsection{Comparison with Finetune}

In Figure~\ref{fig:optimize_vs_finetune}, we compare the advantages of our approach to fine-tuning. 
Specifically, we extract a mini-batch from the dataset and, after each optimization iteration, validate the overall compression rate on the entire \texttt{code} dataset. 
Our findings show that fine-tuning tends to overfit when the size of data is limited. 
In contrast, our method converges within 10 iterations, even with a single data point. 
This suggests that compressing and validating the compression rate across the entire dataset is unnecessary; instead, a near-optimal $\alpha$ can be determined using minimal data, thus avoiding the overfitting issues that commonly arise with fine-tuning.

From a computational efficiency perspective, under the same experimental conditions where a data point is extracted from the OOD \texttt{code} dataset for optimization, Table~\ref{tab:efficiency_table} shows that wPoE optimization outperforms or remains on par with fine-tuning across three different model sizes. 
Even when the computational cost of fine-tuning is slightly higher than that of wPoE optimization, our approach remains more efficient.

\subsection{Stability of Single Data Point Optimization}

In our previous experiments, we optimized the model using only a single data point at test time. This design choice was made to steer the model’s behavior with minimal computational overhead. However, this raises a natural concern: is a single data point sufficiently representative to yield stable and reliable performance?

To address this question, we investigated how the optimized $\alpha$ and the resulting compression rate evolve as the number of samples increases. As shown in Table~\ref{tab:stability}, both the optimized $\alpha$ and the achieved compression rate progressively converge to the values obtained via grid search. Furthermore, we observe a reduction in variance across runs as the sample size increases, indicating improved stability. Nevertheless, even with a single data point, we still observe stable improvements compared to the pretrained model. For each experiment, we report results averaged over ten random seeds.

These findings collectively demonstrate that while single-point optimization is effective and efficient, using more samples can further enhance stability and performance, yielding results that are closer to optimal.

\begin{table}[!t]
  \centering
  \caption{Compression rates and alpha values on the \texttt{code} dataset using the Transformer 200K model pretrained on \texttt{enwik8}, evaluated under varying adaptation sample sizes.}
  \label{tab:stability}
  \small
  \resizebox{\linewidth}{!}{
    \begin{tabular}{lcc}
        \toprule
        \textbf{Model} & \textbf{Compression Rate} & \textbf{Alpha} \\
        \midrule
        Pretrain & 65.67\% & N/A \\
        Grid search & 53.90\% & 0.630 \\
        1 sample  & 54.35 $\pm$ 0.665\% & 0.592 $\pm$ 0.168 \\
        10 samples  & 53.94 $\pm$ 0.028\% & 0.613 $\pm$ 0.033 \\
        100 samples  & 53.92 $\pm$ 0.008\% & 0.609 $\pm$ 0.088 \\
        1000 samples  & 53.92 $\pm$ 0.006\% & 0.609 $\pm$ 0.005 \\
        \bottomrule
    \end{tabular}
  }
\end{table}
\section{Conclusion}

In this work, we introduced a novel test-time steering framework for lossless text compression based on the Weighted Product of Experts (wPoE). By integrating a pretrained autoregressive language model with a training-free universal compressor, \ie, a Naive Bayes classifier with Laplace smoothing, our method effectively enhances out-of-distribution performance without the need for additional retraining. The wPoE framework guarantees that the integrated model performs at least as well as its best individual expert.

Our extensive experiments across multiple datasets demonstrate that our approach consistently improves compression rates, particularly on OOD data where neural compressors typically struggle. Notably, we observed that even with a very limited amount of data, our optimization procedure converges to a near-optimal weight within just 10 iterations when using a second-order optimizer.
Moreover, by progressively incorporating additional experts, we further boost performance, with the inclusion of the Naive Bayes classifier contributing significantly to compression improvements.

These findings underscore the potential of the Weighted Product of Experts (wPoE). For future work, we believe that incorporating other suitable experts could further enhance generalization across various NLP tasks.  
More specifically, if an effective prior can be identified for a given task, integrating it with a pretrained model could improve performance with minimal additional cost.

\newpage
\section{Limitation}
Our model has only one learnable parameter, which limits its expressive capacity. If computational resources are abundant, fine-tuning a pretrained model or training a new model with the same size from scratch on an out-of-distribution (OOD) dataset would lead to better performance.

Moreover, the weighted product of experts (wPoE) approach relies on the experts being sufficiently diverse. If two models are too similar, the stronger model may effectively ``absorb'' the weaker one, causing the weight assigned to the weaker model to approach zero. As a result, the weaker model would contribute little to the overall performance of wPoE.

\section*{Acknowledgement}

This work was funded, in part, by the NSERC DG Grant (No. RGPIN-2019-05448, No. RGPIN-2021-04012, and No. RGPIN-2022-04636), the Vector Institute for AI, Canada CIFAR AI Chair, and a Google Gift Fund. Resources used in preparing this research were provided, in part, by the Province of Ontario, the Government of Canada through the Digital Research Alliance of Canada \url{alliance.can.ca}, and companies sponsoring the Vector Institute \url{www.vectorinstitute.ai/#partners}, and Advanced Research Computing at the University of British Columbia. 
Additional hardware support was provided by John R. Evans Leaders Fund CFI grant.
\bibliography{custom}

\appendix

\section{Appendix}
\label{sec:appendix}
\subsection{Detailed Introduction to Arithmetic Coding}
\label{sec:detailed_intro_to_ac}
Given sequence \(X_{<n+1}\) and an autoregressive model $ p_{\theta}(X_{<n+1}) \;=\; \displaystyle\prod_{i=1}^{n} p_{\theta}(X_i \mid X_{<i})$, the arithmetic code for this sequence is represented by a binary decimal \(\lambda \in [0,1)\). The coding procedure iteratively refines an interval \([l_n, u_n)\subset [0,1)\) such that any real number within the final interval shares the same binary representation up to sufficient precision. This representation is the arithmetic code for \(x_{<n+1}\). 

Set the initial interval 
$
I_0 \;=\; [0,\, 1).
$
Denote \(I_{i-1} = [l_{i-1},\, u_{i-1})\) and its length by \(L(I_{i - 1}) = u_{i-1} - l_{i-1}\).
At the \(i\)-th step (where \(1 \le i \le n\)), we aim to encode the token \(x_i\). Suppose the vocabulary (or set of possible symbols) is \(\{a_1, a_2, \dots, a_D\}\). We divide \(I_{i-1}\) into \(D\) sub-intervals
$
I_i^j,j \in \{1,2,\dots,D\},
$
where each sub-interval \(I_i^j\) has length
$
L(I_i^j)=L(I_{i - 1}) \cdot p_{\theta}(X_i = a_j | x_{<i}).
$
Hence, the sub-interval \(I_i^j\) is defined by
\begin{equation}
    \small
    \begin{aligned}
     &I_i^j :=[l_{i-1} + L(I_{i - 1})\sum_{k=1}^{j-1}p_{\theta}(X_i = a_k \mid x_{<i}),\\
     &l_{i-1} + L(I_{i - 1})\sum_{k=1}^{j}p_{\theta}(X_i = a_k \mid x_{<i})).
    \end{aligned}
\end{equation}

After observing \(x_i = a_j\), we choose 
$
I_i \;=\; I_i^j.
$
Repeating this process for \(i=1,2,\dots,n\) yields the final interval 
$
I_n = [l_n,\, u_n).
$
Any \(\lambda \in [l_n,\, u_n)\) shares the same binary expansion up to the necessary precision. This binary expansion is the adaptive arithmetic code for the sequence \(x_{<n+1}\).
\subsection{Proof of Proposition~\ref{prop:wPoe_ensures_a_better_model}}
\label{sec:proof_of_proposition}
Let  $p_{\theta_1},p_{\theta_2},...,p_{\theta_K}$ be $K$ autoregressive models used to compress a sequence $x_{<n+1} = \{x_1, x_2, \dots, x_n\}$, where $X_{<n} \sim p_{\text{data}}$. Each $ x_i $ takes values from the dictionary $\mathcal{A} = \{ a_1, \dots, a_D \} $.  For an autoregressive model $p_{\theta_k}$, the following equation reveals the relationship between the joint distribution of \(X_{<n}\) and the conditional distribution of \(X_n\):
\begin{equation}
\label{eq:chain_rule_k}
    p_{\theta_k}(X_{<n+1}) = \displaystyle\prod_{i = 1}^{n} p_{\theta_k}(X_i | X_{<i})
\end{equation}
Therefore, the cross-entropy between $p_{\text{data}}$ and a certain model \(p_{\theta_k}\) can be expanded using the Equation~\ref{eq:chain_rule_k} as follows:

\begin{equation}
    H(p_{\text{data}},p_{\theta_k}) = \underset{p_{\text{data}}}{\mathbb{E}} \sum_{i = 1}^{n} -\log p_{\theta_k}(X_i | X_{<i})
\end{equation}
Our weighted product of expert (wPoE) model is given by ~\eqref{eq:wpoe},
\begin{equation*}
\begin{array}{cc}
    p_{\boldsymbol{\theta}, \boldsymbol{\alpha}}(X_n|  X_{<n}) =  \frac{1}{Z(\boldsymbol{\theta}, \boldsymbol{\alpha},n)}\displaystyle\prod_{k = 1}^K p_{\theta_k}(X_n  |  X_{<n})^{\alpha_k},
\end{array}
\end{equation*}
where the weights are $\boldsymbol{\alpha} = \{\alpha_1,...,\alpha_K\}$, $\alpha_k \in[0,1]$, $\sum_{k = 1}^K \alpha_k = 1$, the parameters of experts are $\boldsymbol{\theta} = \{\theta_1,...,\theta_K\}$, and the normalization constant is $Z(\boldsymbol{\theta}, \boldsymbol{\alpha}, n) = \sum_{a \in \mathcal{A}}\prod_{k = 1}^K p_{\theta_k}(X_n = a |  X_{<n})^{\alpha_k}$.

Here we can derive:
\begin{equation}
    \begin{array}{cc}
     H(p_{\text{data}}, p_{\boldsymbol{\theta}, \boldsymbol{\alpha}}) &=  \sum_{k = 1}^K \alpha_k H(p_{\text{data}},p_{\theta_k})\\
     &+\underset{p_{\text{data}}}{\mathbb{E}} \displaystyle\sum_{i=1}^{n} \log \left[Z(\boldsymbol{\theta}, \boldsymbol{\alpha},i)\right].
\end{array}
\end{equation}
To complete the proof, we introduce the following technical lemma for bounding $Z(\boldsymbol{\theta}, \boldsymbol{\alpha},i)$. 
\begin{lemma}
    \label{lemma:proof_inequality}
    Let $p^{(k)} = \bigl(p^{(k)}_1, \ldots, p^{(k)}_D\bigr)$ for $k=1,\dots,K$ 
    be $K$ categorical distributions, so $\sum_{j=1}^D p^{(k)}_j = 1$ for each $k$. 
    Let $\alpha_1,\dots,\alpha_K \ge 0$ satisfy $\sum_{k=1}^K \alpha_k = 1.$ 
    Then
    \[
    \sum_{j=1}^D 
    \prod_{k=1}^K \bigl(p^{(k)}_j\bigr)^{\alpha_k}
    \;\;\le\;\; 1,
    \]
    with equality if and only if $p^{(1)} = p^{(2)} = \cdots = p^{(K)}$ 
    or exactly one $\alpha_k=1$ and the rest are zero.
\end{lemma}
From the Lemma \ref{lemma:proof_inequality}, it can be concluded that:
\begin{align}
    Z(\boldsymbol{\theta}, \boldsymbol{\alpha},i) \leq  1, \forall \boldsymbol{\theta}, \boldsymbol{\alpha},i.
\end{align}
Equality holds if and only if  each distribution $p_{\theta_k}(X_i \mid X_{<i})$ is the same, or $\alpha_k = 1$ and others are 0.
Thus we can conclude that:

\begin{small}
\begin{equation}
\begin{array}{cc}
     \displaystyle\inf_{\boldsymbol{\alpha}}H(p_{\text{data}},p_{\boldsymbol{\theta}, \boldsymbol{\alpha}} ) &\leq \displaystyle\min_{k \in \{1,\dots,K\}} H(p_{\text{data}}, p_{\theta_k}) \\
     &+ \underset{p_{\text{data}}}{\mathbb{E}} \displaystyle\sum_{i=1}^{n} \log \left[Z(\boldsymbol{\theta}, \boldsymbol{\alpha},i)\right]\\
     \displaystyle\inf_{\boldsymbol{\alpha}}H(p_{\text{data}},p_{\boldsymbol{\theta}, \boldsymbol{\alpha}} ) &\leq \displaystyle\min_{k \in \{1,\dots,K\}} H(p_{\text{data}}, p_{\theta_k})
\end{array}
\end{equation}
\end{small}
To complete the proof of Proposition~\ref{prop:wPoe_ensures_a_better_model}, it now suffices to show Lemma~\ref{lemma:proof_inequality}. 
In order to establish Lemma \ref{lemma:proof_inequality}, we first need a helper lemma stated as follows.
\begin{lemma}
    \label{lemma:proof_inequality_2distribution}
    Let $p = (p_1, \ldots, p_D)$ and $q = (q_1, \ldots, q_D)$ be two categorical 
    distributions satisfying $\sum_{j=1}^D p_j = 1$ and $\sum_{j=1}^D q_j = 1$. 
    Then, for any $\alpha \in [0,1]$, the following inequality holds:
    \[
    \sum_{j=1}^D p_j^{\alpha} q_j^{\,1-\alpha} \;\;\le\;\; 1.
    \]
    Moreover, equality holds if and only if $p_j = q_j$ for all $j$ 
    (i.e., $p = q$) or $\alpha \in \{0,1\}$.
\end{lemma}

\begin{proof}[Proof of Lemma~\ref{lemma:proof_inequality_2distribution}]

\textbf{Step 1. Pointwise inequality.} 
For any $x,y \ge 0$ and $\alpha \in [0,1]$, we show
\[
\alpha\,x \;+\; (1-\alpha)\,y \;\;\ge\;\; x^{\alpha}\,y^{\,1-\alpha}.
\]
When $x = 0$, the statement is trivial, so assume $x > 0$. 
Define $t := \frac{y}{x} > 0$. Then the inequality is equivalent to
\[
\alpha \;+\; (1-\alpha)\,t \;\;\ge\;\; t^{\alpha}.
\]
Set $f(t) = t^{\alpha} - (1-\alpha)\,t - \alpha$. We compute its derivative:
\[
\frac{d f}{dt} = \alpha\,t^{\alpha-1} - (1-\alpha).
\]
We have $\frac{d f}{dt}\big|_{t=1} = 0$. One checks that $t=1$ maximizes $f$, 
so $f(t)\le f(1)=0$. Hence
\[
t^{\alpha} \;\le\; \alpha \;+\; (1-\alpha)\,t,
\]
i.e., $x^{\alpha}\,y^{\,1-\alpha} \le \alpha\,x + (1-\alpha)\,y$.

\textbf{Step 2. Summation.} 
Applying the above inequality to each pair $(x,y) = (p_j,q_j)$ and summing 
over $j=1,\ldots,D$ gives
\[
\sum_{j=1}^D p_j^{\alpha} q_j^{\,1-\alpha}
\;\;\le\;\;
\sum_{j=1}^D \bigl(\alpha\,p_j + (1-\alpha)\,q_j\bigr).
\]
Since $\sum_{j=1}^D p_j = 1$ and $\sum_{j=1}^D q_j = 1$, the right side 
simplifies to $\alpha \cdot 1 + (1-\alpha)\cdot 1 = 1$. Hence
\[
\sum_{j=1}^D p_j^{\alpha} q_j^{\,1-\alpha} \;\;\le\;\; 1.
\]

\textbf{Step 3. Equality condition.} 
By the analysis of $f(t)$, the equality in the pointwise inequality holds 
if and only if $t=1$ (i.e., $p_j = q_j$) or $\alpha \in \{0,1\}$. Thus the 
sum only achieves equality if $p=q$ or $\alpha\in\{0,1\}$. 
\end{proof}

We can now prove Lemma~\ref{lemma:proof_inequality} with the help of Lemma~\ref{lemma:proof_inequality_2distribution}.
\begin{proof}[Proof of Lemma~\ref{lemma:proof_inequality}]

\textbf{Base Case ($K=2$).} 
For two distributions $(p_j)$ and $(q_j)$ with weights $\alpha$ and $1-\alpha$, 
it is already proved that
\[
\sum_{j=1}^D p_j^{\,\alpha} \, q_j^{\,1-\alpha} \;\;\le\;\; 1,
\]
with equality precisely if $p=q$ or $\alpha \in \{0,1\}$.

\textbf{Inductive Step.} 
Assume the statement holds for $K-1$ distributions. We prove it for $K$.

Let $p^{(1)}, p^{(2)}, \dots, p^{(K)}$ be $K$ distributions with weights 
$\alpha_1,\dots,\alpha_K$ such that $\sum_{k=1}^K \alpha_k = 1$. 
Define $\alpha' = \alpha_1 + \alpha_2$. 
By the $K=2$ base case,
\[
\sum_{j=1}^D 
\bigl(p^{(1)}_j\bigr)^{\frac{\alpha_1}{\alpha'}}
\bigl(p^{(2)}_j\bigr)^{\frac{\alpha_2}{\alpha'}}
\;\;\le\;\; 1.
\]
Given this sum can not be zero, define a new ``composite'' distribution 
$r = (r_1,\dots,r_D)$ by
\[
r_j 
= 
\frac{
 \bigl(p^{(1)}_j\bigr)^{\frac{\alpha_1}{\alpha'}}\,\bigl(p^{(2)}_j\bigr)^{\frac{\alpha_2}{\alpha'}}
}{
 \sum_{i=1}^D 
 \bigl(p^{(1)}_i\bigr)^{\frac{\alpha_1}{\alpha'}}\,\bigl(p^{(2)}_i\bigr)^{\frac{\alpha_2}{\alpha'}}
},
\quad
j=1,\dots,D.
\]
Clearly, $\sum_{j=1}^D r_j = 1$. 
Now we have $K-1$ distributions: $r, p^{(3)}, \dots, p^{(K)}$, 
with weights $\alpha',\alpha_3,\dots,\alpha_K$ summing to 1.

By the inductive hypothesis, 
\[
\sum_{j=1}^D
r_j^{\,\alpha'}
\,\bigl(p^{(3)}_j\bigr)^{\alpha_3}
\,\cdots\,
\bigl(p^{(K)}_j\bigr)^{\alpha_K}
\;\;\le\;\; 1.
\]
Substituting the definition of $r_j$ into the left-hand side shows
\begin{multline}
\sum_{j=1}^D
  (p^{(1)}_j)^{\alpha_1}(p^{(2)}_j)^{\alpha_2}
  \prod_{k=3}^K (p^{(k)}_j)^{\alpha_k} \\
\le
\Biggl(
  \sum_{i=1}^D
    (p^{(1)}_i)^{\alpha_1}(p^{(2)}_i)^{\alpha_2}
\Biggr)^{\alpha'}.
\end{multline}
Since the base case says $\sum_{i=1}^D (p^{(1)}_i)^{\alpha_1} (p^{(2)}_i)^{\alpha_2} 
\le 1$, 
raising it to the power $\alpha'$ also yields a value at most 1. Thus
\[
\sum_{j=1}^D
\prod_{k=1}^K \bigl(p^{(k)}_j\bigr)^{\alpha_k}
\;\;\le\;\; 1,
\]
which completes the inductive step.

\textbf{Equality Condition.} 
In base case, equality demands $p=q$ or 
$\alpha=0 $ or $\alpha = 1$. Propagating this requirement through 
each step of the recursion shows all distributions must be identical 
or all but one weight are zero. 

By induction, the proof is complete.
\end{proof}

\subsection{Implementation details}
During pretraining, we employ the Adam optimizer \cite{adam} with a learning rate of 5e-4, a batch size of 128, and 30,000 training iterations. 
For optimization with LBFGS \cite{lbfgs}, we use a tolerance of 1e-5 for both the gradient and parameter changes, setting the initial learning rate to 0.5. 
All experiments are conducted on a single A100 GPU with 80GB.
For all our implementations, we use PyTorch\citep{paszke2019pytorch} which is released under BSD 3-Clause license.

\subsection{Additional Discussions}

\subsubsection{Naive Bayes with Lidstone Smoothing}
\label{subsec:lidstone}

We consider the whole family of Naive Bayes with Lidstone Smoothing, where the token probability is
\[
\hat{p}(x) = \frac{c(x) + \alpha}{N + \alpha |V|}, \quad \alpha > 0,
\]
with $c(x)$ the token count, $N$ the total count, and $|V|$ the vocabulary size. This estimator can be written as an interpolation between the empirical Maximum Likelihood Estimation and a uniform distribution:
\[
\hat{p}(x)
= \frac{N}{N+\alpha|V|}\,\frac{c(x)}{N}
+ \frac{\alpha|V|}{N+\alpha|V|}\,\frac{1}{|V|}.
\]

When $\alpha = 0$, it recovers the unsmoothed Maximum Likelihood Estimation. When $\alpha = \tfrac{1}{2}$, it corresponds to the Krichevsky--Trofimov (KT) estimator. When $\alpha = 1$, it represents Naive Bayes with Laplace smoothing.

We sweep $\alpha$ and report compression on Enwik8. KT ($\alpha{=}0.5$) is slightly better than Laplace ($\alpha{=}1$), but smaller $\alpha$ is not uniformly better: as $\alpha$ decreases, the estimator adheres more closely to training counts and can become less robust under sparsity or domain shift. Since $\alpha$ behaves like a dataset-dependent hyperparameter and our goal is simplicity, we fix $\alpha{=}1$ in the main experiments.

\begin{table}[!htb]
  \centering
  \caption{Compression rates on datasets \texttt{enwik8}, as $\alpha$ are changed from 1.0 to 0.01 and context length changed from 1 to 2048.}
  \label{tab:muti_experts_compression_results}
  \resizebox{0.95\linewidth}{!}{
    \begin{tabular}{lcccc}
    \toprule
    \textbf{Context} & \textbf{$\alpha=1.0$} & \textbf{$\alpha=0.5$} & \textbf{$\alpha=0.1$} & \textbf{$\alpha=0.01$} \\
    \textbf{length} & \textbf{(Laplace)} & \textbf{(KT)} & & \\
    \midrule
    1 & 99.38\% & 99.12\% & 98.38\% & 101.50\% \\
    2 & 99.12\% & 98.62\% & 97.88\% & 103.75\% \\
    4 & 97.88\% & 96.88\% & 94.75\% & 101.88\% \\
    8 & 95.88\% & 93.75\% & 89.50\% & 96.62\% \\
    16 & 92.12\% & 88.75\% & 82.12\% & 87.25\% \\
    32 & 87.12\% & 82.62\% & 74.88\% & 78.00\% \\
    64 & 81.25\% & 76.12\% & 69.00\% & 70.75\% \\
    128 & 75.50\% & 70.75\% & 65.50\% & 66.50\% \\
    256 & 70.75\% & 67.12\% & 63.62\% & 64.50\% \\
    512 & 67.75\% & 65.25\% & 63.00\% & 63.62\% \\
    1\,024 & 66.38\% & 64.38\% & 62.75\% & 63.38\% \\
    2\,048 & 66.00\% & 64.25\% & 62.75\% & 63.25\% \\
    \bottomrule
    \end{tabular}
  }
\end{table}

\subsubsection{Beyond a Single Global $\alpha$}
\label{subsec:dynamic-alpha}

A natural question is whether a single global weight $\alpha$ is too restrictive when the data distribution shifts in complex, context-dependent ways. We explored a dynamic alternative by predicting a per-token $\alpha$ with a small MLP placed on top of the frozen backbone's latent representation (the base model and expert remain frozen; only the MLP is trained). This delivered consistent and modest gains over a single global $\alpha$.

\begin{table}[!htb]
  \centering
  \caption{Compression performance comparison between fixed and dynamic weight approaches.}
  \label{tab:multi_experts_compression_results}
  \resizebox{\linewidth}{!}{%
    \begin{tabular}{lcccc}
      \toprule
      \textbf{Model} & \textbf{math} & \textbf{code} & \textbf{shakespeare} & \textbf{enwik8*} \\
      \midrule
      Transformer 200K & 56.25\% & 65.67\% & 44.04\% & 31.59\% \\
      Transformer 200K + single $\alpha$ & 50.95\% & 53.94\% & 42.12\% & 31.58\% \\
      Transformer 200K + per-token $\alpha$ & \textbf{49.23\%} & \textbf{51.39\%} & \textbf{41.56\%} & \textbf{31.54\%} \\
      \bottomrule
    \end{tabular}%
  }
\end{table}

Although a context-aware $\alpha$ could offer finer control, our choice of a single scalar reflects a deliberate trade-off: while it has lower capacity than full finetuning, it is appealing in resource-constrained and real-time settings (e.g., limited-memory GPUs or quantized models that cannot be finetuned). Moreover, much of the expressive power in our approach comes from the additional expert itself, which offsets the simplicity of a global $\alpha$.

\end{document}